\documentclass[10pt,twocolumn,letterpaper]{article}

\usepackage{amsmath,amsfonts,bm}

\def\1{\bm{1}}

\def\vmu{{\bm{\mu}}}
\def\vtheta{{\bm{\theta}}}

\def\vh{{\bm{h}}}

\def\vx{{\bm{x}}}
\def\vy{{\bm{y}}}

\def\evh{{h}}

\def\mG{{\bm{G}}}

\def\mX{{\bm{X}}}

\DeclareMathAlphabet{\mathsfit}{\encodingdefault}{\sfdefault}{m}{sl}
\SetMathAlphabet{\mathsfit}{bold}{\encodingdefault}{\sfdefault}{bx}{n}

\newcommand{\E}{\mathbb{E}}

\newcommand{\R}{\mathbb{R}}

\newcommand{\Var}{\mathrm{Var}}

\DeclareMathOperator*{\argmin}{arg\,min}

\usepackage{cvpr}
\usepackage{times}
\usepackage{epsfig}
\usepackage{graphicx,overpic}
\usepackage{amsmath}
\usepackage{amssymb}
\usepackage{multirow}
\usepackage{booktabs}
\usepackage{bbding}
\usepackage{cite}
\usepackage{makecell}
\usepackage{warpcol}
\usepackage{hyperref}
\usepackage{url}
\usepackage{amsfonts}
\usepackage{amsthm}
\usepackage{bm}
\usepackage{xcolor}
\usepackage{arydshln} 
\usepackage{algorithm}
\usepackage{algorithmic}

\newtheorem{thm}{Theorem}
\newtheorem{defi}{Definition}

\def\ie{\textit{i.e.,~}}

\def\etal{{\em et al.}}

\newcommand{\figref}[1]{Fig.~\ref{#1}}
\newcommand{\tabref}[1]{Tab.~\ref{#1}}
\newcommand{\equref}[1]{Eqn.~\ref{#1}}
\newcommand{\secref}[1]{Sec.~\ref{#1}}
\newcommand{\thmref}[1]{Thm.~\ref{#1}}
\newcommand{\algref}[1]{Alg.~\ref{#1}}
\newcommand{\stepref}[1]{Step~\ref{#1}}

\usepackage{xcolor, colortbl}
\usepackage{arydshln}
\definecolor{hyperref-blue}{RGB}{0,0,200}
\definecolor{hyperref-green}{RGB}{0,150,0}
\definecolor{hyperref-red}{RGB}{200,0,0}
\hypersetup{
    breaklinks=true,
    citecolor=hyperref-blue,
    linkcolor=hyperref-red,
    urlcolor=hyperref-blue,
    letterpaper=true,
    colorlinks=true,
}
\cvprfinalcopy % *** Uncomment this line for the final submission

\def\cvprPaperID{1833} % *** Enter the CVPR Paper ID here

\ifcvprfinal\fi
\begin{document}

%%%%%%%%% TITLE
\title{AdaSample: Adaptive Sampling of Hard Positives for Descriptor Learning}

\author{Xin-Yu Zhang\textsuperscript{1} \quad
Le Zhang\textsuperscript{2} \quad
Zao-Yi Zheng\textsuperscript{1} \quad
Yun Liu\textsuperscript{1} \quad
Jia-Wang Bian\textsuperscript{3} \quad
Ming-Ming Cheng\textsuperscript{1} \quad \\
\textsuperscript{1}TKLNDST, CS, Nankai University~~~~ 
\textsuperscript{2}I2R, A*STAR~~~~
\textsuperscript{3}The University of Adelaide\\
% Institute for Infocomm Research，Agency for Science, Technology and Research (A*STAR)\\
}

\maketitle
%\thispagestyle{empty}

%%%%%%%%% ABSTRACT
\begin{abstract}
Triplet loss has been widely employed in a wide range of
computer vision tasks, including local descriptor learning.
The effectiveness of the triplet loss heavily relies on the
triplet selection, in which a common practice is to first sample intra-class patches (positives)
from the dataset for batch construction
and then mine in-batch negatives to form triplets.
For high-informativeness triplet collection, 
researchers mostly focus on mining hard negatives in the second stage,
while paying relatively less attention to constructing informative batches.
To alleviate this issue, we propose \textit{AdaSample}, 
an adaptive online batch sampler, in this paper.
Specifically, hard positives are sampled based on their \textit{informativeness}.
In this way, we formulate a hardness-aware positive mining pipeline within a novel 
\textit{maximum loss minimization} training protocol.
The efficacy of the proposed method is evaluated on several standard benchmarks,
where it demonstrates a significant and consistent performance gain 
on top of the existing strong baselines.
The source codes will be released upon acceptance.
\end{abstract}

%%%%%%%%% BODY TEXT
%%%%%%%%%%%%%%%%%%%%%%%%%%%%%%%%%%%%%%%%%%%%%%%%%%%%%%%%%%%%%%%
\section{Introduction}
Learning discriminative local descriptors from image patches
is a fundamental ingredient of various computer vision tasks,
including structure-from-motion~\cite{agarwal2009building}, 
image retrieval~\cite{philbin2010descriptor},
and panorama stitching~\cite{brown2007automatic}.
%
% Image matching is a fundamental problem in the computer vision community.
% Typically, image matching can be decoupled into \textit{key-point detection} 
% and \textit{feature descriptor matching}.
% Learning discriminative descriptors is extensively studied in pioneering
% works \cite{lowe2004distinctive,tian2017l2,mishchuk2017working},
% and can benefit a wide range of computer vision applications,
% including structure-from-motion~\cite{agarwal2009building}, 
% image retrieval~\cite{philbin2010descriptor}, 
% and panorama stitching~\cite{brown2007automatic}.
%
Conventional approaches mostly utilize hand-crafted descriptors, such as SIFT \cite{lowe2004distinctive},
which have been successfully employed in a variety of applications.
Recently, with the emergence of large-scale annotated datasets \cite{brown2011discriminative,balntas2017hpatches},
data-driven methods have started to demonstrate their effectiveness,
and learning-based descriptors have gradually dominated this field.
Specifically, convolutional neural network (CNN) based descriptors
\cite{han2015matchnet,zagoruyko2015learning,yi2016lift,tian2017l2,mishchuk2017working,sosnet2019cvpr} can achieve state-of-the-art performance 
on various tasks, including patch retrieval and 3D reconstruction.

Notably, \textit{triplet loss} is adopted in many well-performing descriptor
learning frameworks.
Nevertheless, the quality of the learned descriptors heavily relies on the triplet selection,
and mining suitable triplets from a large database is a challenging task.
Towards this challenge, Balntas \etal~\cite{balntas2016learning} propose an in-triplet hard negative mining strategy called anchor swapping.
Tian \etal~\cite{tian2017l2} progressively sample unbalanced training pairs in favor of negatives,
and Mishchuk \etal~\cite{mishchuk2017working} further simplify this idea to mine the 
hardest negatives within the mini-batch.
Despite the significant progress on performance and generalization ability,
however, two potential problems still exist in the current hardest-in-batch sampling solution:
i) hard negatives are mined in the batch level, while randomly selected 
matching pairs can still be easily discriminated by the descriptor network;
% batches may still lead to sub-optimal sampling solution;
ii) it does not take the interaction between the training progress 
and the hardness of the training samples into consideration.
To this end, we propose a novel triplet mining pipeline
to adaptively construct high-informativeness batches in a principled manner.

Our proposed method is nominated as \textit{AdaSample}, 
in which matching pairs are sampled from the dataset based on 
their \textit{informativeness} to construct mini-batches.
The methodology is developed on informativeness analysis,
where \textit{informativeness} is defined via the contributing
gradients of the potential samples and can assist estimate 
their optimal sampling probabilities.
Moreover, we propose a novel training protocol inspired by 
\textit{maximum loss minimization}~\cite{ShalevShwartz2016MinimizingTM} 
to boost the generalization ability of the descriptor network.
Under this training framework, we can adaptively adjust the overall hardness
of the training examples fed to the network, based on the training progress.
Comprehensive evaluation results and ablation studies on several standard benchmarks \cite{brown2011discriminative,balntas2017hpatches} demonstrate the effectiveness of our proposed method.

In summary, our contributions are three-fold:
\begin{itemize}
    \item We theoretically analyze the \textit{informativeness} of potential
    training examples and formulate a principled sampling approach for
    descriptor learning.
    
    \item We propose a hardness-aware training protocol inspired by
    \textit{maximum loss minimization}, in which the overall hardness
    of the generated triplets are adaptively adjusted to match the training progress.
    
    \item Comprehensive evaluation results on popular benchmarks demonstrate
    the efficacy of our proposed \textit{AdaSample} framework.
\end{itemize}

\section{Related work}
\paragraph{Local Descriptor Learning.}
Traditional descriptors \cite{lowe2004distinctive,ke2004pca} mostly utilize hand-crafted features
to extract low-level textures from image patches.
The seminal work, \ie SIFT \cite{lowe2004distinctive}, computes the smoothed weighted histograms
using the gradient field of the image patch.
PCA-SIFT \cite{ke2004pca} further improves the descriptors by applying
Principle Component Analysis (PCA) to the normalized image gradient.
A comprehensive overview of the hand-crafted descriptors can be found in \cite{mikolajczyk2005performance}.

Recently, due to the rapid development of deep learning,
CNN-based methods enable us to learn feature descriptors directly from the raw image patches.
MatchNet~\cite{han2015matchnet} propose a two-stage Siamese architecture 
to extract feature embeddings and measure patch similarity,
which significantly improves the performance and demonstrates
the great potential of CNNs in descriptor learning.
DeepDesc \cite{simo2015discriminative} trains the network with Euclidean distance
and adopts a mining strategy to sample hard examples.
DeepCompare \cite{zagoruyko2015learning} explores various architectures of the
Siamese network and develops a two-stream network focusing on image centers.

With the advances of metric learning, triplet-based architectures have gradually
replaced the pair-based ones.
TFeat \cite{balntas2016learning} adopts the triplet loss and mines in-triplet 
hard negatives with a strategy named anchor swapping.
$\rm L2$-Net \cite{tian2017l2} employs progressive sampling and requires that
matching patches have minimal $\rm L2$ distances within the mini-batch.
HardNet \cite{mishchuk2017working} further develops the idea to mine the hardest-in-batch
negatives with a simple triplet margin loss.
DOAP \cite{he2018local} imposes a ranking-based loss directly optimized for the average precision.
GeoDesc \cite{luo2018geodesc} further incorporates the geometric constraints 
from multi-view reconstructions and achieves significant improvement on 3D 
reconstruction task.
SOSNet~\cite{sosnet2019cvpr} proposes a second-order similarity regularization term
and achieves more compact patch clusters in the feature space.
A very recent work \cite{zhang2019learning} relaxes the hard margin in the triplet margin loss 
with a dynamic soft margin to avoid manually tuning the margin by human heuristics.

From previous arts, we find that the triplet mining framework
can generally be decoupled into two stages,
\ie batch construction from the dataset and triplet generation within the mini-batch.
Previous works \cite{balntas2016learning,tian2017l2,mishchuk2017working} mostly
focus on mining hard negatives in the second stage,
while neglecting batch construction in the first place.
Besides, their sampling approaches do not take the training progress into consideration
when generating triplets.
Therefore, we argue that their triplet mining solutions still cannot exploit the full
potential of the entire dataset to produce triplets with suitable hardness.
To alleviate this issue, we analyze the contributing gradients of 
the potential training examples and sample informative matching pairs for
batch construction.
Then, we propose a hardness-aware training protocol inspired by \textit{maximum loss minimization},
in which the overall hardness of the selected triplets is correlated with the
training progress.
Incorporating the hardest-in-batch negative mining solution, 
we formulate a powerful triplet mining framework, \textit{AdaSample},
for descriptor learning, in which the quality of the learned
descriptors can be significantly improved by a simple sampling strategy.

\paragraph{Hard Negative Mining.}
Hard negative mining has been widely used in deep metric learning, 
such as face verification \cite{schroff2015facenet},
as it can progressively select hard negatives for triplet loss and Siamese networks 
to boost the performance and speed up the convergence.
FaceNet \cite{schroff2015facenet} samples semi-hard triplets
within the mini-batch to avoid overfitting
the outliers.
Wu \etal~\cite{wu2017sampling} select training examples based on their relative distances.
Zheng \etal~\cite{zheng2019hardness} augment the training data by adaptively synthesizing hardness-aware and label-preserving examples.
However, our sampling solution differs from them in that 
we analyze the \textit{informativeness} of the training data
and ensure that the sampled data can provide gradients contributing 
most to the parameter update.
Besides, our method can adaptively adjust the hardness of the selected training data
as training progresses.
In this way, well-classified samples are filtered out, 
and the network is always fed with informative triplets with suitable hardness.
Comprehensive evaluation results demonstrate consistent performance improvement
contributed by our proposed approach.

\section{Methodology}
\subsection{Problem Overview}
Given a dataset that consists of $N$ 
classes\footnote{
The term "class" stands for the image patches that come from the same 3D location.
For our sampling purpose, patches from a single class are matching, while non-matching pairs come from different classes.
}
with each containing 
$k$ matching patches, we decompose the triplet generation into two stages.
Firstly, we select $n$ matching pairs (positives) to form a mini-batch, 
where $n$ is the batch size.
This is done by our proposed \textit{AdaSample}, as introduced in \secref{sec:adasample}.
Secondly, we mine the hardest-in-batch negatives for each matching pair
and use the triplet loss to supervise the network training, as in \secref{sec:hardest-in-batch}.
See \figref{fig:netplot} for an illustration of the two-stage
sampling pipeline.
Finally, the overall solution is summarized in \secref{sec:angular}.

\begin{figure}[!tb]
  \centering
  \begin{overpic}[width=1\linewidth]{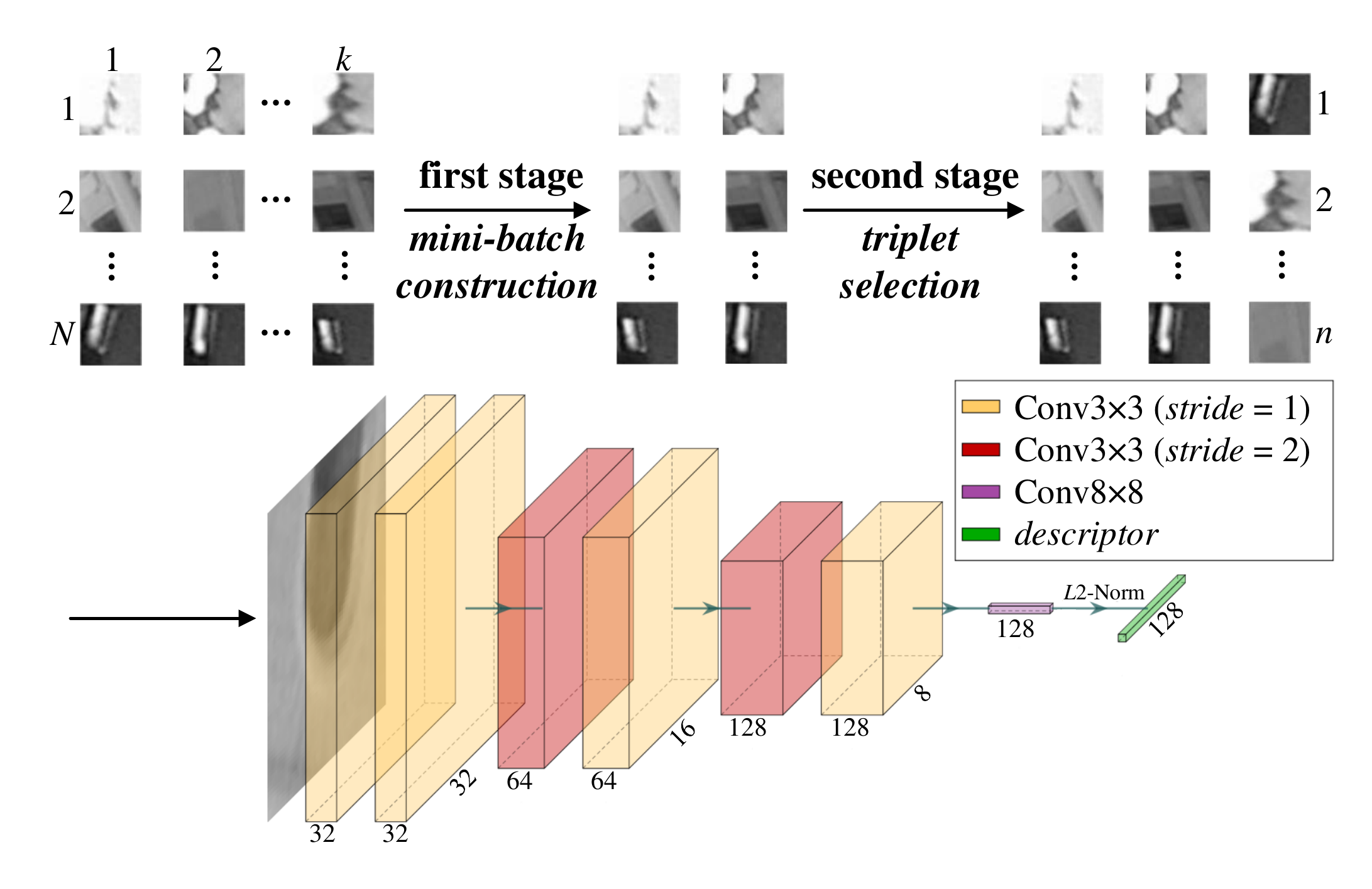}
  \end{overpic}
  \vspace{-3mm}
  \caption{Illustration of the two-stage descriptor learning pipeline.}
  \label{fig:netplot}
  \vspace{-4mm}
\end{figure}

\subsection{AdaSample}\label{sec:adasample}
Previous works~\cite{tian2017l2,mishchuk2017working} sample positives
randomly to construct mini-batches,
yielding a majority of similar matching pairs which can be 
easily discriminated by the network.
This practice may reduce the overall hardness of the triplets.
Motivated by the hardest-in-batch mining strategy in \cite{mishchuk2017working},
a straightforward solution is to select the most dissimilar matching pairs.
However, potential issues arise, 
\ie the network may be trained with bias in favor of 
the most dissimilar matching pairs, while other cases are less-considered.
We validate this solution, nominated as
\textit{Hardpos}, in experiments (\secref{sec:ablation}).

A more principled solution is to sample positives based on their \textit{informativeness}.
Here, we assume that informative pairs are those contributing
most to the optimization, namely, providing 
effective gradients for parameter updates.
Therefore, we quantify the \textit{informativeness} of 
the matching pairs by measuring their contributing
gradients during training.
Moreover, we employ \textit{maximum loss minimization}
\cite{ShalevShwartz2016MinimizingTM} to improve the generalization
ability of the learned model and show that the resulting
gradient estimator is an \textit{unbiased} estimator of the actual gradient.
In the following, 
we introduce our derivation and elaborate on the theoretical justification in~\secref{sec:theorey}.

\vspace{-3mm}
\paragraph{Informativeness Based Sampling.}
In the end-to-end deep learning literature, 
the training data contribute to optimization via gradients, 
so we measure the \textit{informativeness} of training examples 
by analyzing their resulting gradients.
Generally, we consider the generic deep learning framework.
Let $(\vx_i,\vy_i)$ be the $i^{th}$ data-label pair of the training set,
$f(\vx; \vtheta)$ be the model parameterized by $\vtheta$, and 
$\mathcal{L}(\cdot, \cdot)$ be a differentiable loss function.
The goal is to find the optimal model parameter $\vtheta^*$ that minimizes the average loss, \ie
\begin{equation}
    \vtheta^* = \argmin_{\vtheta} \frac{1}{K}\sum_{i=1}^K \mathcal{L}(f(\vx_i; \vtheta), \vy_i), \tag{1}\label{equ:theta-star}
\end{equation}
where $K$ denotes the number of training examples.
Then, we proceed with the following definition of \textit{informativeness}.
\begin{defi}
The informativeness of a training example $(\vx_i, \vy_i)$ is quantified by its resulting
gradient norm at iteration $t$, namely,
\begin{equation}
    {\rm info}(\vx_i, \vy_i) := || \nabla_{\vtheta_t} \mathcal{L}(f(\vx_i; \vtheta), \vy_i) ||_{2}.\tag{2}\label{equ:informativeness}
\end{equation}
\end{defi}
At iteration $t$, let $P_t = \{p_1^t, \dots, p_K^t\}$ be the sampling probabilities of each datum in the training set.
More generally, we also re-weight each sample by $w_1^t,\cdots,w_K^t$.
Let random variable $I_t$ denote the sampled index at iteration $t$, 
then $I_t \sim P_t$, namely, $\mathbb{P}(I_t = i) = p_i^t$.
We record the re-weighted gradient
induced by the training sample $(\vx_i,\vy_i)$ as
\begin{equation}
    \mG_i^t = w_i^t \nabla_{\vtheta_t} \mathcal{L}(f(\vx_i;\vtheta), \vy_i).\tag{3}\label{equ:g}
\end{equation}
For simplicity, we omit the superscript $t$ when no ambiguity is made.
By setting $w_i = \frac{1}{K p_i}$,
we can make the gradient estimator $\mG_i$ an unbiased estimator of the actual gradient, \ie
\begin{equation}
    \E_{I_t \sim P_t} \left[ \bm{G}_{I_t} \right] = \nabla_{\vtheta_t} \frac{1}{K}\sum_{i=1}^K \mathcal{L}(f(\vx_i;\vtheta), \vy_i).\tag{4}\label{equ:unbiasedness}
\end{equation}
Without loss of generality, we use stochastic gradient descent (SGD)
to update model parameters:
\begin{equation}
    \vtheta_{t+1} = \vtheta_t - \eta_t w_{I_t} \nabla_{\vtheta_t} \mathcal{L}(f(\vx_{I_t};\vtheta), \vy_{I_t}) = \vtheta_t - \eta_t \bm{G}_{I_t},\tag{5}\label{equ:sgd}
\end{equation}
where $\eta_t$ is the learning rate at iteration $t$.
As the goal is to find the optimal $\vtheta^*$, 
we define the expected progress towards the optimum 
at each iteration as follows.
\begin{defi}
At iteration $t$, the expected parameter rectification $R_t$ is defined as the expected reduction of the squared distance between the parameter $\vtheta$ and the optimum $\vtheta^*$ after iteration $t$,
\begin{equation}
    R_t := -\E_{I_t \sim P_t} \left[ || \vtheta_{t+1} - \vtheta^* ||_2^2 - || \vtheta_{t} - \vtheta^* ||_2^2 \right]. \tag{6}\label{equ:rt}
\end{equation}
\end{defi}
Generally, tens of thousands of iterations are included in the training so that
the empirical average parameter rectification will converge to the average of $R_t$ asymptotically.
Therefore, by maximizing $R_t$, we guarantee the most progressive step 
towards parameters optimum at each iteration in the expectation sense.
Inspired by the \textit{greedy algorithm} \cite{edmonds1971matroids},
we aim to maximize $R_t$ at each iteration.

It can be shown that maximizing $R_t$ is equivalent to minimizing $tr(\Var \left[ \bm{G}_{I_t} \right])$ (\thmref{thm:rt}).
Under this umbrella, we show that the optimal sampling probability is
proportional to the per-sample gradient norm
(a special case of \thmref{thm:optimal-probability}).
Therefore, the optimal sampling probability of each datum happens
to be proportional to its \textit{informativeness}.
This property justifies our definition of \textit{informativeness} as the resulting gradient norm of each training example.

However, as the neural network has multiple layers with a large 
number of parameters, it is computationally prohibitive to
calculate the full gradient norm.
Instead, we prove that the matching distance in the feature space
is a good approximation to the
\textit{informativeness}\footnote{The approximation is up to a constant factor, 
which is insignificant as it will be offset by the learning
rate. The same reasoning applies to the approximation of gradients in \textbf{Maximum Loss Minimization} paragraph.}
in \secref{sec:approximation}.
Concretely, for each class consisting of $k$ patches $\{\mX_i : i=1, \dots, k\}$,
we first select a patch $\mX_{i_0}$ randomly,
which serves as the anchor patch,
and then sample a matching patch $\mX_{i}$ with probability
\begin{equation}
    p_{i} \propto d(\vx_{i}, \vx_{i_0}), ~~\textit{for}~i \not= i_0, \tag{7}\label{equ:approximation}
\end{equation}
where $\vx_i$ is the extracted descriptor of $\mX_i$, and $d(\cdot, \cdot)$ measures the discrepancy of the extracted descriptors. See specific choices of $d$ in \secref{sec:angular}.

\begin{algorithm*}[!ht]
\caption{Pipeline of \textit{AdaSample} framework.} 
\label{alg:adasample} 
\begin{algorithmic}[1] 
\REQUIRE ~~\\ 
Dataset of $N$ classes with each containing $k$ matching patches;\\
Moving average of history loss $\mathcal{L}_{avg}$;\\
Hyperparameter $\lambda$;
\STATE Randomly select $n$ distinct classes from the dataset without replacement; 
\label{step:sample-id}
\STATE Extract descriptors of the patches belonging to the selected classes;
\label{step:extract}
\FOR{each selected class with $k$ patches $\{\mX_i : i=1, \dots, k\}$}
\STATE Sample an anchor patch $\mX_{i_0}$ randomly; 
\label{step:sample-anchor}
\STATE Sample a matching patch $\mX_{i}$ from the remaining patches with probabilities specified by \equref{equ:adasample};
\label{step:sample-with-prob}
\ENDFOR
\STATE With sampled positive pairs and their descriptors $\{(\Tilde{\vx}_i,\Tilde{\vx}_i^+)\}_{i=1}^n$, compute Angular Triplet Hinge loss by \equref{equ:angular}; 
\label{step:angular}
\STATE Backpropagate and update model parameters via $\sum_{i=1}^{n} w_i \nabla_{\vtheta} \mathcal{L}_i$; 
\label{step:backprop}
\end{algorithmic}
\end{algorithm*}

\paragraph{Maximum Loss Minimization.}
Minimizing the average loss may be sub-optimal because the training tends to be overwhelmed 
by well-classified examples that provide noisy gradients \cite{lin2017focal}.
On the contrary, well-classified examples can be adaptively
filtered out by minimizing the maximum 
loss \cite{ShalevShwartz2016MinimizingTM},
which can further improve the generalization ability.
However, directly minimizing the maximum loss may lead to
insufficient usage of training data and sensitivity to outliers,
so we approximate the gradient of maximum loss by 
$\nabla_{\vtheta_t} \frac{1}{K} \sum_{i=1}^K \mathcal{L}_i^{\alpha}$,
in which $\alpha$ is sufficiently large.
As $\bm{G}_{I_t}$ is used to update parameters, consider its expectation
\begin{equation}
    \E_{I_t \sim P_t} \left[ \bm{G}_{I_t} \right] = 
    \E_{I_t \sim P_t} \left[w_{I_t} \nabla_{\vtheta_t} \mathcal{L}_{I_t} \right] = 
    \sum_{i=1}^K p_i w_i \nabla_{\vtheta_t} \mathcal{L}_i. \tag{8}\label{equ:unbiasedness-non-uniform}
\end{equation}
To guarantee $\bm{G}_{I_t}$ is an unbiased 
estimator\footnote{
We impose the unbiasedness constraints due to its theoretical
convergence guarantees.
For example, the non-asymptotic error bound induced by unbiased gradient estimates
is referred to \cite{moulines2011non}.
For re-weighted SGD, as in our case, improved convergence rate can be found in \cite{needell2014stochastic}.
}
of $\nabla_{\vtheta_t} \frac{1}{K} \sum_{i=1}^K \mathcal{L}_i^{\alpha}$, 
it suffices to set
\begin{equation}
    p_i w_i = \frac{\alpha}{K} \mathcal{L}_i^{\alpha - 1},
    \tag{9}\label{equ:joint-control}
\end{equation}
as in this case,
\begin{equation}
    \E_{I_t \sim P_t} \left[ \bm{G}_{I_t} \right] = 
    \sum_{i=1}^K \frac{\alpha}{K} \mathcal{L}_i^{\alpha - 1} \nabla_{\vtheta_t} \mathcal{L}_i =
    \sum_{i=1}^K \frac{1}{K} \nabla_{\vtheta_t} \mathcal{L}_i^\alpha. \tag{10}\label{equ:unbiasedness-non-uniform-again}
\end{equation}
Following the previous reasoning, we need to minimize $tr(\Var \left[ \bm{G}_{I_t} \right])$
under the constraints specified by \equref{equ:joint-control}
in order to step most progressively at each iteration.
In \thmref{thm:optimal-probability}, we show that the optimal sampling probability and re-weighting scalar should be given by
\begin{equation}
    p_i \propto \mathcal{L}_i^{\alpha - 1} || \nabla_{\vtheta_t} \mathcal{L}_i ||_2~~\textit{and}~~w_i \propto || \nabla_{\vtheta_t} \mathcal{L}_i ||_2^{-1}.\tag{11} \label{equ:optimal-probability}
\end{equation}
As previously claimed, we approximate the gradient norm
via the matching distance in the feature space.
Besides, in our case, the hinge triplet loss (\equref{equ:triplet-loss}) is
positively (or even linearly) correlated
with the matching distance squared.
Therefore, we use the matching distance squared as an approximation 
of the hinge triplet loss.
Thus, the sampling probability and re-weighting scalar are given by
\begin{equation}
    p_i \propto d(\vx_i, \vx_{i_0})^{2 \alpha - 1}~~\textit{and}~~w_i \propto d(\vx_i, \vx_{i_0})^{-1}, ~~\textit{for}~i \not= i_0. \tag{12}\label{equ:optimal-probability-special}
\end{equation}

Moreover, for better approximation, it is preferable to adjust 
$\alpha$ adaptively, namely, to increase $\alpha$ with training.
Intuitively, when easy matching pairs have been 
correctly classified, we focus more on hard ones.
A good indicator of the training progress is the \textit{average loss}.
As a result, instead of pre-defining a sufficiently large $\alpha$, 
we set $2\alpha - 1 = \lambda / \mathcal{L}_{avg}$, 
where $\lambda$ is a tunable hyperparameter, and $\mathcal{L}_{avg}$
is the moving average of history loss. 
Formally, we formulate our sampling probability and re-weighting scalar as
\begin{equation}
    p_i \propto d(\vx_i, \vx_{i_0})^{\frac{\lambda}{\mathcal{L}_{avg}}}~~\textit{and}~~w_i \propto d(\vx_i, \vx_{i_0})^{-1}, ~~\textit{for}~i \not= i_0. \tag{13}\label{equ:adasample}
\end{equation}
The exponent increases adaptively as training progresses so that hardness-aware training examples can be generated and fed 
to the network.
Our sampling approach is thus named as \textit{AdaSample}.

\subsection{Triplet generation by hardest-in-batch}\label{sec:hardest-in-batch}
\textit{AdaSample} focuses on the batch construction stage,
and for a complete triplet mining framework,
we need to mine negatives from the mini-batch as well.
Here, we adopt the hardest-in-batch strategy in~\cite{mishchuk2017working}.
Formally, given a mini-batch of $n$ matching pairs
$\{ (\widetilde{\mX}_i, \widetilde{\mX}_i^+) : i = 1, \dots, n \}$,
let $(\widetilde{\vx}_i, \widetilde{\vx}_i^+)$ be
the descriptors extracted from
$(\widetilde{\mX}_i, \widetilde{\mX}_i^+)$\footnote{For clarity, $(\widetilde{\mX}_{\diamond}, \widetilde{\mX}_{\diamond}^+)$ denotes the selected matching pairs,
with different pairs belonging to different classes.
$\mX_{\diamond}$ denotes a generic patch in a specific class, 
where $\diamond$ denotes the placeholder for the index.}.
For each matching pair $(\widetilde{\mX}_i, \widetilde{\mX}_i^+)$, 
we select the non-matching patch which lies closest 
to one of the matching patches in the feature space.
Then, the Hinge Triplet (HT) loss is defined as follows:
\begin{align*}
    \mathcal{L}_i &= \max{\{t+(d_i^{pos})^2-(d_i^{neg})^2,~0\}},\\
    d_i^{pos} &= d(\widetilde{\vx}_i, \widetilde{\vx}_i^+),\\
    d_i^{neg} &= \min_{j \not= i} { \left\{ \min \left\{ d(\widetilde{\vx}_i, \widetilde{\vx}_j),~ d(\widetilde{\vx}_i^+, \widetilde{\vx}_j^+) \right\} \right\} },\tag{14}\label{equ:triplet-loss}\\
\end{align*}
where $t$ denotes the margin.
Incorporating the re-weighting scalar, we update the model parameter
via the gradient estimator
$\sum_{i=1}^{n} w_i \nabla_{\vtheta} \mathcal{L}_i$.

\subsection{Distance Metric}\label{sec:angular}
Euclidean distance is widely used in previous works
\cite{simo2015discriminative,tian2017l2,mishchuk2017working,sosnet2019cvpr}.
However, as the descriptors lie on the unit hypersphere 
in $128$-dimensional space (\secref{sec:details}),
it is more natural to adopt the geodesic distance of the embedded manifold.
Therefore, we adopt the angular distance~\cite{deng2018arcface} as follows:
\begin{equation}
    d(\widetilde{\vx}_1, \widetilde{\vx}_2) = \arccos(\widetilde{\vx}_1 \bm{\cdot} \widetilde{\vx}_2),\tag{15}\label{equ:angular}
\end{equation}
where $\bm{\cdot}$ denotes the inner product operator.
We nominate our loss function as Angular Hinge Triplet (AHT) loss, 
which is demonstrated to result in consistent performance
improvement (\secref{sec:ablation}).

\algref{alg:adasample} summarizes the overall triplet generation framework.
For each training iteration, 
we first randomly pick $n$ distinct classes from the dataset
and extract descriptors for patches belonging to these classes
(\stepref{step:sample-id}, \ref{step:extract}).
Then, we randomly choose a patch as the anchor from each of 
the selected classes (\stepref{step:sample-anchor})
and adopt our proposed \textit{AdaSample} to select
an informative matching patch (\stepref{step:sample-with-prob}).
With the generated mini-batch, we mine hard negatives following
\cite{mishchuk2017working} and compute Angular Hinge
Triplet (AHT) loss (\stepref{step:angular}). 

\section{Theoretical Analysis}\label{sec:theorey}
In this section, we complete the theoretical analysis of
\textit{informativeness} in \secref{sec:importance},
and prove that the matching distance can serve as a good approximation
of \textit{informativeness} in \secref{sec:approximation}.

\subsection{Informativeness Formulation}\label{sec:importance}
Following notations in \secref{sec:adasample},
we reformulate $R_t$ (\equref{equ:rt}),
and give an equivalent condition for maximizing $R_t$.
The same conclusion can be found in \cite{katharopoulos2018not}.
\begin{thm}\label{thm:rt}
    Let $R_t$, $\vtheta^*$, and $\bm{G}_i$ be defined as in \equref{equ:rt}, \ref{equ:theta-star} and \ref{equ:g}, respectively. Then, we have
    \begin{align*}
        R_t =& 2 \eta_t (\vtheta_t - \vtheta^*)^T \E_{I_t \sim P_t} \left[ \bm{G}_{I_t} \right] \\
        &- \eta_t^2 \E_{I_t \sim P_t} \left[ \bm{G}_{I_t} \right]^T \E_{I_t \sim P_t} \left[ \bm{G}_{I_t} \right] - 
        \eta_t^2 tr(\Var \left[ \bm{G}_{I_t} \right]).\tag{16}\label{equ:thm1}
    \end{align*}
\end{thm}
Due to unbiasedness (\equref{equ:unbiasedness}), 
the first two terms in \equref{equ:thm1} is fixed, 
so maximizing $R_t$ is equivalent to 
minimizing $tr(\Var \left[ \bm{G}_{I_t} \right])$.
\thmref{thm:optimal-probability} specifies the optimal probabilities
to minimize the aforementioned trace under a more general assumption.
\begin{thm}\label{thm:optimal-probability}
    Let $\bm{G}_i$ be defined in \equref{equ:g} and suppose the sampled index $I_t$ obeys distribution $P_t$. 
    Then, given the constraints $p_i w_i = \frac{\alpha}{K} \mathcal{L}_i^{\alpha - 1}$,
    $tr(\Var \left[ \bm{G}_{I_t} \right])$ is minimized by the following optimal sampling probabilities:
    \begin{equation}
        p_i = \frac{1}{Z} \mathcal{L}_i^{\alpha - 1} || \nabla_{\vtheta_t} \mathcal{L}_i ||_2,
        ~\textit{where}~Z = \sum_{j=1}^{K} \mathcal{L}_j^{\alpha - 1} || \nabla_{\vtheta_t} \mathcal{L}_j ||_2.
        \tag{17}\label{equ:optimal-prob-thm}
    \end{equation}
\end{thm}

\begin{proof}
    As $\bm{G}_{I_t}$ is an \textit{unbiased} estimator of the actual gradient (\equref{equ:unbiasedness}), $\E_{I_t \sim P_t} \left[ \bm{G}_{I_t} \right]$ is fixed in our case, denoted by $\vmu$ for short.
    By the linearity of trace and $tr(\vmu \vmu^T) = || \vmu ||_2^2$,
    we have
    \begin{align*}
        tr(\Var \left[ \bm{G}_{I_t} \right])
        &= tr(\E_{I_t \sim P_t} \left[ (\bm{G}_{I_t} - \vmu) (\bm{G}_{I_t} - \vmu)^T \right])\\
        &= tr(\E_{I_t \sim P_t} \left[ \bm{G}_{I_t} \bm{G}_{I_t}^T - \vmu \vmu^T \right])\\
        &= \E_{I_t \sim P_t} \left[ tr(\bm{G}_{I_t} \bm{G}_{I_t}^T) \right] - tr(\vmu \vmu^T)\\
        &= \E_{I_t \sim P_t} \left[ \left\| \bm{G}_{I_t} \right\|_2^2 \right] - || \vmu ||_2^2\\
        &= \sum_{i=1}^{K} p_i w_i^2 \left\| \nabla_{\vtheta_t} \mathcal{L}_i \right\|_2^2 - || \vmu ||_2^2\\
        &= \frac{\alpha^2}{K^2} \sum_{i=1}^{K} \frac{\mathcal{L}_i^{2 \alpha - 2} \left\| \nabla_{\vtheta_t} \mathcal{L}_i \right\|_2^2}{p_i} - || \vmu ||_2^2.
        \tag{18} 
    \end{align*}
    Mathematically, given the constraints $\sum_{i=1}^{K} p_i = 1$,
    the aforementioned harmonic mean of $\{ p_1, \dots, p_K \}$ reaches its minimum when the probabilities satisfy
    \begin{equation}
        p_i \propto \mathcal{L}_i^{\alpha - 1} \left\| \nabla_{\vtheta_t} \mathcal{L}_i \right\|_2.
        \tag{19}
    \end{equation}
    Dividing by a normalization factor, we get the expression in \equref{equ:optimal-prob-thm}.
\end{proof}

Note that in the special case of $\alpha = 1$,
the constraints degrade into $p_i w_i = \frac{1}{K}$,
and the optimal sampling probabilities become $p_i \propto || \nabla_{\vtheta_t} \mathcal{L}_i ||_2$.

\subsection{Approximation of Informativeness}\label{sec:approximation}
As mentioned in \secref{sec:adasample}, the matching distance 
can serve as a good approximation of \textit{informativeness}.
We justify this here.
For simplicity, we introduce some notations for a $L$-layer multi-layer perceptron (MLP).
Let $\vtheta^{(l)} \in \R^{M_l \times M_{l-1}}$ be the weight matrix
for layer $l$ and $g^{(l)}$ be a Lipschitz continuous activation function.
Then the multi-layer perceptron can be formulated as follows:
\begin{align*}
    \vx^{(0)} &= \vx,\\
    \vh^{(l)} &= \vtheta^{(l)} \vx^{(l-1)},~~\textit{for}~l = 1, \cdots, L,\\
    \vx^{(l)} &= g^{(l)}(\vh^{(l)}),~~\textit{for}~l = 1, \cdots, L,\\
    f(\vx;\vtheta) &= \vx^{(L)},\\
    \vtheta &= \{\vtheta^{(1)}, \dots, \vtheta^{(L)}\}. \tag{20} \label{equ:mlp} \\
\end{align*}

Note that although our notations describe only MLPs without bias, 
our analysis holds for any affine transformation
followed by a Lipschitz continuous non-linearity.
Therefore, our reasoning can naturally extend to CNNs.
With
\begin{align*}
    \bm{\Gamma}_l(\vh^{(l)}) &= \text{diag}\left\{ g'^{(l)}(\evh_1^{(l)}), \dots, g'^{(l)}(\evh_{M_l}^{(l)}) \right\},\\
    \bm{\Pi}^{(l)} &= \bm{\Gamma}_l(\vh^{(l)}) \vtheta_{l+1}^T \cdots \bm{\Gamma}_{L-1}(\vh^{(L-1)}) \vtheta_{L}^T \bm{\Gamma}_{L}(h^{(L)}), \tag{21} \label{equ:mlp-grad}
\end{align*}
we have
\begin{align*}
    || \nabla_{\vtheta_l} \mathcal{L}(f(\vx;\vtheta), \vy) ||_2 
    &= \left\| \left( \bm{\Pi}^{(l)} \nabla_{\vx^{(L)}} \mathcal{L} \right) \left( \vx^{(l-1)} \right)^T \right\|_2 \\
    &\leq || \bm{\Pi}^{(l)} ||_2 || \vx^{(l-1)} ||_2 || \nabla_{\vx^{(L)}} \mathcal{L} ||_2. \tag{22} \label{equ:mlp-gradnorm}
\end{align*}
Various data preprocessing,
weight initialization \cite{glorot2010understanding,he2015delving},
and activation normalization \cite{ioffe2015batch,lei2016layer,ulyanov2016instance}
techniques uniformize the activations of each layer across samples.
Therefore, the variation of gradient norms is mostly captured
by the gradient of the loss function \textit{w.r.t.} the output
of neural networks,
\begin{equation}
    {\rm info}(\vx, \vy) =
    || \nabla_{\vtheta} \mathcal{L}(f(\vx; \vtheta), \vy) ||_2 
    \approx M \left\| \nabla_{\vx^{(L)}} \mathcal{L} \right\|_2, \tag{23} \label{equ:info-approximation}
\end{equation}
where $M$ is a constant,
and $ M \left\| \nabla_{\vx^{(L)}} \mathcal{L} \right\|_2$ serves 
as a precise approximation of the full gradient norm. 
For simplicity, we consider hinge triplet loss
(\equref{equ:triplet-loss}) here.
Then, the gradient norm \textit{w.r.t.} the descriptor
of the matching patch is just twice the matching 
distance\footnote{This relation holds only when the hinge triplet loss is positive. Empirically, due to the relatively large margin, the hinge loss never becomes zero.},
\begin{equation}
    \left\| \nabla_{\vx^{(L)}} \mathcal{L} \right\|_2 = 2d^{pos}. \tag{24}
\end{equation}
As a result, we reach the conclusion that
the matching distance is a good approximation
to the \textit{informativeness}.
Also, we empirically verify this in \secref{sec:ablation}.

\section{Experiments}
\begin{table*}[t]
    \centering
    \scalebox{0.95}{ 
    \begin{tabular}{lcccccccccc}
         \toprule[2pt]
         \multirow{2}{*}{Descriptor} & \multirow{2}{*}{Length}
         & {\small Train$\rightarrow$} & Notredame & Yosemite & Liberty
         & Yosemite & Liberty & Notredame & \multirow{2}{*}{Mean} \\
         \cline{4-5} \cline{6-7} \cline{8-9}
         & & {\small Test$\rightarrow$} &\multicolumn{2}{c}{Liberty}
         & \multicolumn{2}{c}{Notredame} & \multicolumn{2}{c}{Yosemite} \\\hline
         SIFT \cite{lowe2004distinctive} & 128 &
         & \multicolumn{2}{c}{29.84} & \multicolumn{2}{c}{22.53}
         & \multicolumn{2}{c}{27.29} & 26.55\\
         DeepDesc \cite{7410379} & 128 &
         & \multicolumn{2}{c}{10.9} & \multicolumn{2}{c}{4.40}
         & \multicolumn{2}{c}{5.69} & 6.99\\
         GeoDesc \cite{luo2018geodesc} & 128 &
         & \multicolumn{2}{c}{5.47} & \multicolumn{2}{c}{1.94}
         & \multicolumn{2}{c}{4.72} & 4.05\\
         MatchNet \cite{han2015matchnet} & 4096 &
         & 7.04 & 11.47 & 3.82 & 5.65 & 11.60 & 8.70 & 8.05 \\
         $\rm L2$-Net \cite{tian2017l2} & 128 &
         & 3.64 & 5.29 & 1.15 & 1.62 & 4.43 & 3.30 & 3.24 \\
         CS-$\rm L2$-Net \cite{tian2017l2} & 256 &
         & 2.55 & 4.24 & 0.87 & 1.39 & 3.81 & 2.84 & 2.61 \\
         HardNet \cite{mishchuk2017working} & 128 &
         & \textbf{1.47} & 2.67 & 0.62 & \textbf{0.88} & 2.14 & 1.65 & 1.57 \\
         HardNet-GOR \cite{zhang2017learning} & 128 &
         & 1.72 & 2.89 & 0.63 & 0.91 & 2.10 & 1.59 & 1.64 \\\cdashline{1-10}[4pt/6pt]
         HardNet* & 128 &
         & 1.80 & 2.89 & 0.68 & 0.90 & 1.93 & 1.71 & 1.65 \\
         AdaSample* (\textbf{Ours}) & 128 &
         & 1.64 & \textbf{2.62} & \textbf{0.61} & \textbf{0.88} & \textbf{1.92} & \textbf{1.46} & \textbf{1.52}\\\hline
         TFeat-M+ \cite{balntas2016learning} & 128 &
         & 7.39 & 10.31 & 3.06 & 3.80 & 8.06 & 7.24 & 6.64 \\
         $\rm L2$-Net+ \cite{tian2017l2} & 128 &
         & 2.36 & 4.70 & 0.72 & 1.29 & 2.57 & 1.71 & 2.23 \\
         CS-$\rm L2$-Net+ \cite{tian2017l2} & 256 &
         & 1.71 & 3.87 & 0.56 & 1.09 & 2.07 & 1.30 & 1.76 \\
         HardNet+ \cite{mishchuk2017working} & 128 &
         & 1.49 & 2.51 & 0.53 & 0.78 & 1.96 & 1.84 & 1.51 \\
         HardNet-GOR+ \cite{zhang2017learning} & 128 &
         & 1.48 & 2.43 & 0.51 & 0.78 & 1.76 & 1.53 & 1.41 \\
         DOAP+ \cite{he2018local} & 128 &
         & 1.54 & 2.62 & 0.43 & 0.87 & 2.00 & 1.21 & 1.45 \\\cdashline{1-10}[4pt/6pt]
         HardNet+* & 128 &
         & 1.32 & 2.31 & 0.41 & 0.67 & 1.51 & 1.24 & 1.24 \\
         AdaSample+* (\textbf{Ours}) & 128 &
         & \textbf{1.25} & \textbf{2.21} & \textbf{0.40} 
         & \textbf{0.63} & \textbf{1.40} & \textbf{1.14} & \textbf{1.17} \\
         \toprule[2pt]
    \end{tabular}}
    \caption{Patch classfication results on UBC Phototour dataset~\cite{brown2011discriminative}. The false positive rate at $95\%$ recall is reported. $+$ indicates data augmentation and $*$ indicates positive generation.}
    \label{tab:brown}
\end{table*}

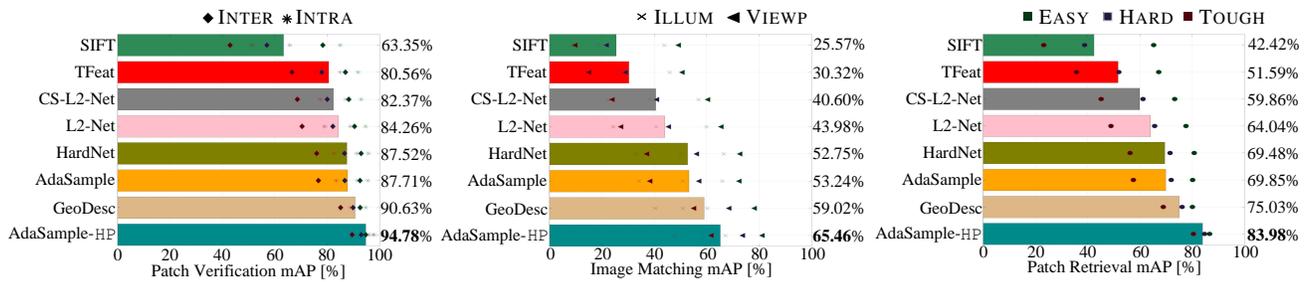
\begin{figure*}[!tb]
  \centering
  \begin{overpic}[width=1\linewidth]{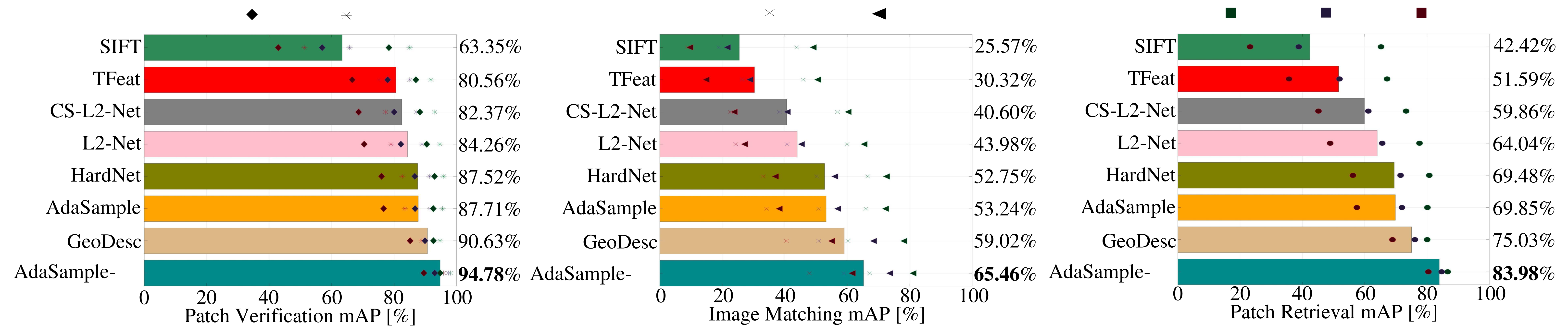}
    \put(16.8,20.1){\footnotesize{\textsc{Inter}}}
    \put(22.7,20.1){\footnotesize{\textsc{Intra}}}
    \put(57,20.1){\footnotesize{\textsc{Viewp}}}
    \put(50,20.1){\footnotesize{\textsc{Illum}}}
    \put(79.3,20.1){\footnotesize{\textsc{Easy}}}
    \put(85.4,20.1){\footnotesize{\textsc{Hard}}}
    \put(91.5,20.1){\footnotesize{\textsc{Tough}}}
    \put(7.45,3.65){{\scriptsize \texttt{HP}}}
    \put(40.36,3.55){{\scriptsize \texttt{HP}}}
    \put(73.35,3.65){{\scriptsize \texttt{HP}}}
  \end{overpic}
  \vspace{-5pt}
  \caption{Evaluation results on HPatches dataset \cite{balntas2017hpatches}.
    By default, descriptors are trained on \textit{Liberty} subset of UBC Phototour~\cite{brown2011discriminative} dataset,
    and ``\texttt{-HP}'' indicates descriptors trained on
    HPatches training set of split a.
    Marker color indicates the level of geometrical noises and
    marker type indicates the experimental setup. 
    \textsc{Inter} and \textsc{Intra} indicate the source of negative examples for the \textit{verification} task.
    \textsc{Viewp} and \textsc{Illum} indicate the sequence type for the \textit{matching} task.
  }\label{fig:hpatches}
  \vspace{-5pt}
\end{figure*}

\subsection{Implementation Details} \label{sec:details}
We adopt the architecture of $\rm L2$-Net \cite{tian2017l2} to embed local 
descriptors into the unit hypersphere in $128$-dimensional space.
Following prior works \cite{tian2017l2,mishchuk2017working},
all patches are resized to $32 \times 32$ and
normalized to zero per-patch mean and unit per-patch variance.
We train our model from scratch in PyTorch library \cite{paszke2017automatic}
using SGD optimizer with initial learning rate $\eta=10$,
momentum $0.5$, and weight decay $0.0001$.
Batch size is $1024$, margin $t = 1$, and $\lambda=10$
unless otherwise specified. 
We generate $1,000,000$ matching pairs for each epoch, and
the total number of epochs is $90$.
The learning rate is divided by $10$ at the end
of $30$, $60$, $80$ epochs.

We compare our method with both handcrafted and deep methods\footnote{
Note that the training dataset of GeoDesc \cite{luo2018geodesc} is not released,
so the comparison may be unfair.
Besides, some recent works \cite{sosnet2019cvpr,zhang2019learning} explore
in different directions, and their training codes are not publicly available.
So we leave the efficacy comparison and system combination in future work.},
including SIFT \cite{lowe2004distinctive}, DeepDesc~\cite{7410379},
TFeat~\cite{balntas2016learning},
$\rm L2$-Net~\cite{tian2017l2}, HardNet~\cite{mishchuk2017working},
HardNet with global orthogonal regularization (GOR) \cite{zhang2017learning}, 
DOAP~\cite{he2018local}, and GeoDesc~\cite{luo2018geodesc}.
Comprehensive evaluation results and ablation studies
on two standard descriptor datasets:
UBC Phototour \cite{brown2011discriminative} (\secref{sec:brown}),
and HPatches \cite{balntas2017hpatches} (\secref{sec:hpatches})
demonstrate the efficacy of our proposed sampling framework.

\subsection{UBC Phototour} \label{sec:brown}

UBC Phototour \cite{brown2011discriminative}, also known as Brown dataset, 
consists of three subsets: 
\textit{Liberty}, \textit{Notre Dame}, and \textit{Yosemite},
with about $\rm 400K$ normalized $64 \times 64$ patches in each subset. 
Keypoints are detected by DoG detector~\cite{lowe2004distinctive} and verified by $\rm 3D$ model.
The testing set consists of $\rm 100K$ matching and non-matching pairs for each sequence.
For evaluation, models are trained on one subset and tested on the other two.
The metric is the false positive rate (FPR) at $95\%$ true positive recall. 
The evaluation results are reported in \tabref{tab:brown}.

Our method outperforms other approaches by a significant margin.
We randomly flip and rotate by $90$ degrees for data augmentation, noted by $+$.
Besides, for our method, we also generate positive patches by random rotation
such that each class has $15$ patches, noted by *. 
We augment matching pairs as there are too few patches (two or three) 
corresponding to one class in UBC Phototour
dataset~\cite{brown2011discriminative},
which limits the capacity of our method.
To analyze its effect, we also conduct it for
HardNet~\cite{mishchuk2017working} baseline.
It can be seen that our method consistently outperforms the baseline,
indicating the effectiveness of our adaptive sampling solution.
% and observe merely a slight performance gain,
% indicating that the performance improvement
% mainly comes from our adaptive sampling solution.

\subsection{HPatches} \label{sec:hpatches}
HPatches \cite{balntas2017hpatches} consists of $116$ sequences of $6$ images.
The dataset is split into two parts: \textit{viewpoint} - $59$ sequences with 
significant viewpoint change and \textit{illumination} - $57$ sequences with 
significant illumination change.
According to the level of geometric noises,
the patches can be further divided into three groups: 
\textit{easy}, \textit{hard}, and \textit{tough}.
There are three evaluation tasks: \textit{patch verification},
\textit{image matching}, and \textit{patch retrieval}.
Following standard evaluation protocols of the dataset,
we show results in \figref{fig:hpatches}.
It demonstrates that our method performs in favor of other methods on
patch verification task, which is consistent with the patch classification
results in \tabref{tab:brown}.
Furthermore, our descriptors achieve the best results on the more challenging
image matching and patch retrieval tasks,
indicating the improved generalization ability contributed by our approach.

\subsection{Ablation Study}\label{sec:ablation}
\paragraph{Informativeness Approximation.}
We empirically verify the conclusion in \secref{sec:approximation} that the probability induced by matching distance approximate well to the one induced by \textit{informativeness} (\figref{fig:empirical}, Left).
Besides, the results show that the Pearson correlation is consistently greater than $0.8$ during training (\figref{fig:empirical}, Right),
which indicates these probabilities have strong correlation with each other statistically.

\begin{figure}[!hbt]
    \centering
    \includegraphics[width=0.47\textwidth]{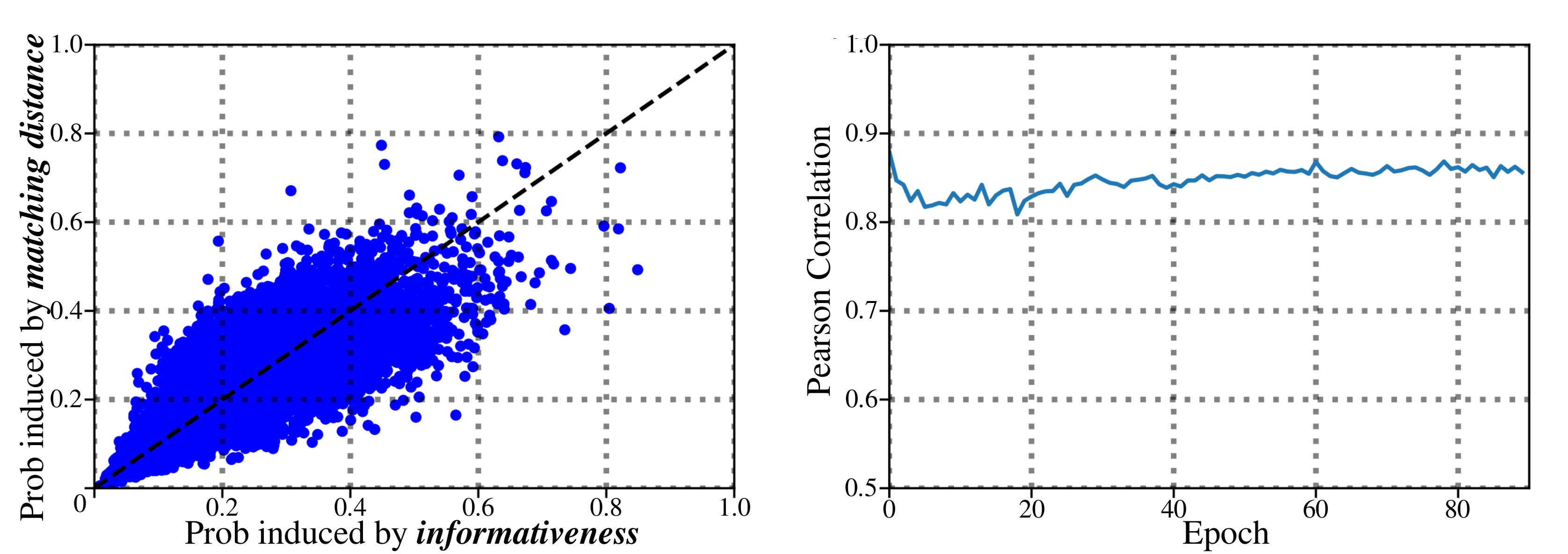}
    \caption{(Left) Probabilities induced by informativeness and matching distance.
    	(Right) Pearson correlation between probabilities and training epochs.}
    \label{fig:empirical}
    \vspace{-4mm}
\end{figure}

\paragraph{Impact of $\lambda$ and Distance Metric.}
We experiment with varying $\lambda$ in \textit{AdaSample} to control the overall hardness of the selected matching pairs.
A large $\lambda$ indicates that hard matching pairs are more likely to be selected.
When $\lambda = 0$, our method degrades into random sampling and the
overall framework becomes HardNet \cite{mishchuk2017working}, 
and as $\lambda \rightarrow +\infty$, the framework becomes \textit{Hardpos}.
Therefore, both HardNet and \textit{Hardpos} are special cases of our proposed \textit{AdaSample}.
\tabref{tab:ablation} shows the results on
HPatches~\cite{balntas2017hpatches} dataset,
where $\lambda=10$ leads to the best results in most cases.
It demonstrates the advantages of our balanced sampling strategy against the hardest solution.
Also, \tabref{tab:ablation} demonstrates that the angular hinge triplet (AHT) loss outperforms the commonly-used hinge triplet (HT) loss in most cases.

\begin{table}[!ht]
    \centering
    \scalebox{0.9}{
    \begin{tabular}{lcccccc}
        \toprule[2pt]
        Task & \multicolumn{2}{c}{Verification} & \multicolumn{2}{c}{Matching} & \multicolumn{2}{c}{Retrieval}\\
        \hline
        Loss & AHT & HT & AHT & HT & AHT & HT\\
        \hline
        $\lambda = 1$ & 93.84 & 93.17 & 64.09 & 62.64 & 81.26 & 79.97\\
        $\lambda = 2$ & 94.72 & 94.56 & \textbf{66.04} & 65.92 & 83.58 & 83.34\\
        $\lambda = 5$ & \textbf{94.78} & 94.76 & 65.89 & 65.68 & 83.80 & 83.54\\
        $\lambda = 10$ & \textbf{94.78} & 94.60 & 65.46 & 65.37 & \textbf{83.98} & 83.62\\
        $\lambda = 20$ & 94.60 & 94.69 & 64.56 & 64.84 & 83.56 & 83.69\\
        $\lambda \rightarrow +\infty$ & 94.42 & 94.51 & 63.81 & 64.02 & 83.41 & 83.29\\
        \toprule[2pt]
    \end{tabular}}
    \caption{Ablation studies on the impact of $\lambda$. All experiments are conducted on HPatches \cite{balntas2017hpatches} benchmark.}
    \label{tab:ablation}
    \vspace{-4mm}
\end{table}

\paragraph{Stability and Reproducibility.}
The sampling naturally comes from stochasticity. 
% one may wonder the reproducibility of our approach.
%
To ensure reproducibility, we conduct experiments on five runs with different
random seeds and show the means and standard deviations of the
patch classification results in \tabref{tab:reproducibility}.
It demonstrates the stability of our sampling solution.
We argue that a possible explanation of the stability is the
\textit{unbiasedness} of the gradient estimator
(\equref{equ:unbiasedness-non-uniform-again}).
As the number of training triplets is huge, the estimated gradients
converge to the actual gradient asymptotically.
Therefore, the gradients can guide the network towards
the parameter optimum as training progresses,
regardless of the specific random condition.

\begin{table}[!ht]
    \centering
    \scalebox{0.86}{
    \begin{tabular}{llcccc}
         \toprule[2pt]
         Train & Test & HardNet+* & AdaSample+* & Rel $\uparrow$ & \textit{p} value \\\hline
         Notr & \multirow{2}{*}{Lib} & 1.316$\pm$0.044 & 1.254$\pm$0.026 & 4.71\% & 0.031 \\
         Yos &  & 2.310$\pm$0.063 & 2.212$\pm$0.049 & 4.28\% & 0.018 \\\hline
         Lib & \multirow{2}{*}{Notr} & 0.406$\pm$0.011 & 0.400$\pm$0.016 & 1.58\% & 0.337 \\
         Yos & & 0.671$\pm$0.010 & 0.627$\pm$0.012 & 6.62\% & 0.006 \\\hline
         Lib & \multirow{2}{*}{Yos} & 1.513$\pm$0.084 & 1.395$\pm$0.050 & 7.80\% & 0.030 \\
         Notr & & 1.241$\pm$0.044 & 1.137$\pm$0.036 & 8.38\% & 0.011 \\
         \toprule[2pt]
    \end{tabular}}
    \caption{Reproducibility and statistical significance of our proposed \textit{AdaSample}.
    The repeated experiments are conducted on UBC Phototour \cite{brown2011discriminative} dataset.
    Here, ``Rel $\uparrow$'' indicates the relative improvement upon the HardNet \cite{mishchuk2017working} baseline.
    }\label{tab:reproducibility}
    \vspace{-3mm}
\end{table}

\paragraph{Statistical Significance.}
Since previous methods have been approaching the saturating point
in terms of the performance on UBC Phototour~\cite{brown2011discriminative} dataset,
it is challenging to make progress on top of the
HardNet~\cite{mishchuk2017working} baseline.
%
% Please note that the local descriptor learning has been an challenging task and we are approaching a saturating point where newer efforts face relatively small performance gains.
However, with the proposed method, we still observe a consistent improvement,
as demonstrated in \tabref{tab:reproducibility}.
It can be seen that our method can give a relative improvement of
up to 8.38\% in terms of patch classification accuracy,
indicating our superiority.
To be more principled, we also demonstrate the statistical significance
of our improvement upon the baseline.
Specifically, we adopt the non-parametric hypothesis testing, \ie
the classic Mann-Whitney testing~\cite{mann1947test},
to test whether a random variable is stochastically
larger than the other one.
In our setting, the two random variables are the performance of
\textit{AdaSample} and HardNet baseline, respectively,
and the null hypothesis is that our method
\textit{cannot} significantly improve the performance.
The \textit{p values} under different experimental settings
are summarized in \tabref{tab:reproducibility}.
With a significance level of $\alpha=5\%$, we can reject the null hypothesis
in 5 of the 6 experiments in total.
For the only anomaly, \ie training on Notredame and testing on Liberty,
we conjecture that the reason lies in the extremely high performance
of the HardNet baseline (about 0.4\% in terms of FPR).
Therefore, we argue that the statistical significance under the
other 5 experimental settings is sufficient to verify the effectiveness of
our approach.
% we conclude that our proposed method can bring statistically
% significant improvement upon the baseline.

\section{Conclusion}
This paper proposes \textit{AdaSample} for descriptor learning, 
which adaptively samples hard positives to construct informative mini-batches during training.
We demonstrate the efficacy of our method from both theoretical and empirical perspectives.
Theoretically, we give a rigorous definition of \textit{informativeness} of potential training examples.
Then, we reformulate the problem and derive a tractable sampling probability expression (\equref{equ:adasample}) to generate hardness-aware
training triplets. 
Empirically, we enjoy a consistent and statistically significant performance gain on top of the HardNet~\cite{mishchuk2017working} baseline 
when evaluated on various tasks,
including patch classification, patch verification, image matching, and patch retrieval.

\clearpage

{\small
\bibliographystyle{ieee_fullname}
\bibliography{ref}
}

\end{document}